\documentclass[10pt,twocolumn,letterpaper]{article}

\usepackage[pagenumbers]{cvpr} 

\usepackage{amsthm}
\usepackage[dvipsnames]{xcolor}

\definecolor{cvprblue}{rgb}{0.21,0.49,0.74}
\usepackage[pagebackref,breaklinks,colorlinks,citecolor=cvprblue]{hyperref}

\def\paperID{14823} 
\def\confName{CVPR}
\def\confYear{2024}

\usepackage{algorithm}
\usepackage{algpseudocode}
                                   
\newcommand*\ouralgo{{\sc MOLERE}}
\newcommand*\ouralgostr{{\bf M}eta-{\bf O}ptimized {\bf LE}arned {\bf RE}weighting}
\makeatletter
\newtheorem*{rep@theorem}{\rep@title}
\newcommand{\newreptheorem}[2]{%
\newenvironment{rep#1}[1]{%
 \def\rep@title{#2 \ref{##1}}%
 \begin{rep@theorem}}%
 {\end{rep@theorem}}}
\makeatother

\newtheorem{theorem}{Theorem}
\newreptheorem{theorem}{Theorem}

\title{Improving Generalization via Meta-Learning on Hard Samples}

\author{Nishant Jain~~~~~~~Arun S. Suggala~~~~~~~Pradeep Shenoy\\
 Google Research India\\
{\tt\small \{nishantjn,arunss,shenoypradeep\}@google.com}\\
}
\begin{document}
\maketitle

\begin{abstract}

\textit{Learned reweighting} (LRW) approaches to supervised learning use an optimization criterion to assign weights for training instances, in order to maximize performance on a representative validation dataset. We pose and formalize the problem of \textit{optimized selection} of the validation set used in LRW training, to improve classifier generalization. In particular, we show that using hard-to-classify instances in the validation set has both a theoretical connection to, and strong empirical evidence of generalization. We provide an efficient algorithm for training this \textit{meta-optimized model}, as well as a simple train-twice heuristic for careful comparative study. We demonstrate that LRW with easy validation data performs consistently worse than LRW with hard validation data, establishing the validity of our meta-optimization problem. Our proposed algorithm outperforms a wide range of baselines on a range of datasets and domain shift challenges (Imagenet-1K, CIFAR-100,  Clothing-1M, CAMELYON, WILDS, etc.), with ~1\% gains using VIT-B on Imagenet. We also show that using naturally hard examples for validation (Imagenet-R / Imagenet-A) in LRW training for Imagenet improves performance on both clean and naturally hard test instances by 1-2\%. Secondary analyses show that using hard validation data in an LRW framework improves margins on test data, hinting at the mechanism underlying our empirical gains. We believe this work opens up new research directions for the meta-optimization of meta-learning in a supervised learning context.

\end{abstract}

\section{Introduction}
Overparameterized models, common in supervised learning~\cite{neyshabur2018role}, carry the risk of overfitting to training data. Typically, model generalization is measured on a validation dataset separate from the training data, for purposes of  hyperparameter selection. Increasingly, this validation dataset is itself used as part of the learning objective in nested formulations, e.g., for hyperparameter tuning via gradient descent on the validation loss~\cite{franceschi2018bilevel}. In particular, learned reweighting (LRW) approaches \textit{learn} importance weights associated with training instances~\cite{shu2019meta,zhang2021learning,ren2018learning, holtz2022learning, jain2024instance, jain2024learning} or groups of training instances~\cite{mohri2019agnostic, zhou2022model} by optimizing a weighted training loss alongside an unweighted \textit{meta-loss} on the validation data. This bilevel optimization aligns training loss with the validation data distribution via reweighting, a useful property in addressing covariate shift~\cite{sugiyama2008direct}, and for group DRO~\cite{zhou2022model, faw2020mix}.  Even for in-domain test data, \citet{ren2018learning} show that using ``clean’’ validation data with bilevel optimization can overcome significant amounts of label noise in training data. 

Thus the \textit{choice of validation set} in a LRW paradigm can greatly influence the quality and properties of the learned classifier. We therefore ask the question: can we \textit{optimize} the choice of validation data in LRW so as to maximize generalization of the resulting classifier? We refer to this problem of validation set selection as \textit{meta-optimization}, since it produces data that is the input of another optimization (the meta-learning approach underlying LRW). The closest related work to our proposal is \cite{zhang2021learning} which constructs a validation dataset on the fly during training of an LRW classifier, with a criterion of choosing representative instances. 

Our primary hypothesis is that we can improve classifier generalization by using a validation set consisting of hard-to-classify instances from the training distribution. We therefore pose and formalize the problem of \ouralgostr\ (\ouralgo) where the partitioning of data into train and validation splits, and the LRW classifier corresponding to that split, are jointly optimized. 
We design an efficient algorithm to tackle this meta-optimization, and make the following contributions:

\begin{itemize}
\item We formalize the  problem statement of \textit{validation set optimization} in learned reweighting (LRW) classifiers for improved generalization. We prove that asymptotically our optimization objective exactly achieves our stated goal of maximizing accuracy on the hardest samples. 
\item  We simplify the nested optimization of our proposal into a tractable bi-level optimization with a min-max game between two auxiliary networks: a ``splitter'' that finds the hardest samples, and a ``reweighter'' that minimizes loss on those samples using LRW. We also provide a simple train-twice heuristic that can be used for careful analysis of the choice of validation data in LRW.
\item We show strict accuracy ordering of LRW models based on validation set: easy $<$ random $<$ hard,  demonstrating the importance of  optimizing LRW validation sets. We obtain reliable gains over ERM across datasets (e.g., 1\% on Imagenet/VIT-B backbone), and in domain-generalization (e.g., 1.36\% on iWildCam dataset \cite{beery2020iwildcam,koh2021wilds}) and noisy-label settings (e.g., 4.2\% on Clothing1M dataset \cite{xiao2015learning}). We outperform a range of baselines including reweighting, meta-learning among others. Analyses show improved margins on test set in \ouralgo\ classifiers, suggesting an explanation of our gains.
 \item We extend our results to natural hard samples as validation (Imagenet with Imagenet-R / Imagenet-A), showing 1-3\% gains on \textit{both in-domain and out-of-domain} test sets. This shows the value of our ideas even in scenarios where we pay no additional cost for meta-optimization.
\end{itemize}

\noindent
We hope that our work will be seen as an initial step in establishing the value of meta-optimization of meta-learning, with our findings providing a strong proof-of-concept for the general research direction. 

\section{Related Work}

\textbf{Importance weighting for robustness.} There have been several works on learning robust representations via example re-weighting lately. Mostly, these work aim at avoiding noisy labels in the train set by analyzing the effect of decreasing loss on a given instance on a clean validation set. This line of work includes either learning a per-instance free parameter \cite{ren2018learning} or learning a simple MLP network \cite{shu2019meta} to predict importance of an instance based on its loss value. However, the requirement of a clean validation set limits their applicability to realistic scenarios. To deal with this, a meta-learning based re-weighting scheme (Fast Sample Reweighting~\cite{zhang2021learning}) was proposed based on generating some sort of pseudo-clean data as a proxy validation set. It also proposed certain approximations to make the training process more computation efficient. On a different note, a recent proposal called RHO-loss~\cite{mindermann2022prioritized} proposed to select only \textit{worthy} points for training which increased generalization property of the model, calculated as the difference of the training loss and a hold-out set loss.

A recent line of work \cite{jain2024instance, jain2024learning} proposed weighting based on context/relevance of an instance compared with the overall data distribution. These papers target slow temporal drift in longitudinal datasets, and the development of better uncertainty measures for selective classification, respectively. In particular, \citet{jain2024learning}, suggest that using target domain data in the validation set can improve domain shift performance of classifiers.

\textbf{Meta Learning.} The sample re-weighting task using a validation set comes under the umbrella of meta-learning \cite{hospedales2020meta}, which follows the \textit{learning-to-learn} paradigm. It is conceptually similar to model agnostic meta learning (MAML) \cite{finn2017model}, which learns a single set of parameters that can easily be customized (few-shot) to multiple tasks. Learning these shared parameters involves a nested optimization similar to the one presented here, and significantly optimized for efficiency by recent work\cite{raghu2019rapid,nichol2018first,zintgraf2018caml}.

\textbf{Probabilistic Margins.}  Recent work \cite{liu2021probabilistic} showed that the probabilistic margin in multi-class problems can be used to improve how neural networks deal with  adversarial examples. The margin is defined as the difference between the probability of the true label and the largest of the remaining label probabilities,  indicating the difficulty of classifying that instance. \citet{liu2021probabilistic}  propose re-weighting training instances, in presence of an adversarial attack, inversely related to their probability margins. 

\textbf{Just Train Twice}. A recent work \cite{liu2021just} proposed an effective strategy to improve sensitivity of ERM models towards certain groups by training the ERM model in 2 stages. The first stage is a standard training procedure, whereas the second stage involves giving more importance to the incorrectly classified examples in the first stage by up-weighting their loss in the aggregate loss term for updating the parameters. This can be interpreted as similar to giving more importance to low margin examples (here essentially a hard separation between positive and negative). \\





\section{Preliminaries: Learned ReWeighting}

We work with learned reweighting (LRW) classifiers where training data is reweighted in order to optimize some specified metric on validation data.  
%
%
%
The basic LRW formulation~\cite{ren2018learning} works with two datasets $S_{tr} = \{(x_i, y_i)\}_{i=1}^N$ and $S_{val}=\{(x_i,y_i)\}_{i=1}^M$  (training \& validation, respectively).  Given a desired loss function $\ell(y, \hat{y})$, LRW learns a classifier $f_\theta(\cdot),$ (with parameters $\theta$) and an instance-wise weighting function $\phi(\cdot)$ that minimize following bi-level objective:
\begin{equation}
\begin{aligned}
\label{eq:bilevel-meta}
    \theta^*(\phi) &= \arg \min_\theta \sum_{(x,y) \in S_{tr}} \phi(x) \ell(y, f_\theta(x)), \\
    \quad \texttt{s.t. } \phi^* &= \arg \min_\phi \sum_{(x,y)\in S_{val}} \ell(y, f_{\theta^*(\phi)}(x))
\end{aligned}
\end{equation}



Notice that the validation loss is unweighted, while the training objective is a weighted loss. Essentially, the above bilevel objective computes an \textit{ estimate} of classifier performance on the validation set, and optimizes it indirectly  via reweighting of training data to influence the learned classifier. The weighted-loss-minimizing classifier $f_{\theta^*}(\cdot)$, and the weights $\phi^*(\cdot)$, are learned jointly, typically using alternating stochastic updates over minibatches to make the learning tractable~\cite{ren2018learning}; this is in line with other meta-learning approaches such as MAML~\cite{finn2017model}.
%
Following more recent work~\cite{jain2024instance}, we use an instance-dependent neural network $\phi(x)$ for learning the training instance weights, instead of free parameters~\cite{ren2018learning}.

\noindent \textbf{Intuition:} LRW has been used for overcoming training label noise~\cite{ren2018learning,shu2019meta} and for handling covariate shift~\cite{zhou2022model,jain2024instance}. In these cases, the validation set is assumed to be representative of test samples (i.e., clean labels, or covariate-shifted data, respectively). The intuition is that the learned reweighting of the training loss \textit{aligns} it with the (possibly different) validation distribution in the process of indirectly minimizing the meta-loss (Eq.~\ref{eq:bilevel-meta}). In particular, \citet{shu2019meta} show analytically that ``validation-like'' instances are upweighted. Thus, the validation set in LRW can be thought of as a \textit{target for generalization}, with the learned classifier being optimized for performance on that target set. 




\section{\ouralgo: Optimizing LRW models}

\subsection{Hypothesis and formal objective}\label{sec:hyp}

Our primary hypothesis is that we can improve generalization capabilities of supervised learning by combining two ingredients: a) a learned-reweighting classifier, as described in the previous section, and b) an \textit{ optimized validation set} that strongly encourages desired properties in the reweighting classifier. We refer to this idea as \textbf{M}eta-\textbf{O}ptimization of the \textbf{Le}arned \textbf{Re}weighting framework. 

In particular, we propose learning an LRW classifier with hard samples as validation set, to improve accuracy and generalization of learned classifiers. Since LRW by definition maximizes classifier performance on a given validation set, we believe that this choice of validation set will maximize generalization. Thus, given a dataset for training a predictive model, we need to 1) select the ``hard instances'' from the dataset and separate it into a validation set, and 2) train a classifier on the remaining data using this validation set for LRW.
Notice that hardness of instances is determined in terms of the learned model itself; thus, formalizing the above idea leads to a joint optimization problem of data partitioning (train, validation), and LRW training. We present the formal problem below.

\noindent \textbf{Objective:} Let $S=\{(x,y)\}_{i=1}^{N+M}$ be the available data, and let $\Theta$ be the splitting function that splits $S$ into training and validation datasets; to be precise, we let $\Theta(S)$ be the validation set and $\Theta(S)^c$ be its complement. MOLERE aims to solve the following tri-level optimization problem
\begin{equation}
\begin{aligned}
\label{eqn:tri_level}
&\Theta^* = \arg \max_\Theta \sum_{(x,y)\in \Theta(S)} \ell(y, f_{\theta^*(\phi^*(\Theta), \Theta)}(x)) \\
   &\texttt{where} \ \phi^*(\Theta) = \arg \min_\phi \sum_{(x,y)\in \Theta(S)} \ell(y, f_{\theta^*(\phi, \Theta)}(x))\\
  &\texttt{s.t.}   \theta^*(\phi, \Theta) = \arg \min_\theta \sum_{(x,y) \in \Theta(S)^c} \phi(x) \ell(y, f_\theta(x)).
\end{aligned}
\end{equation}

    
In words: Find a data split, such that across all possible splits the LRW classifier learned on the split has maximal error on the chosen validation set. In practice, we impose an additional constraint on the validation set size: $|\Theta(S)|/(N+M)\leq\delta$, where $\delta$ is some predefined fractional constant. 

\subsection{\ouralgo\ objective and generalization}\label{sec:thm}
 We now study the asymptotic properties of our proposed meta-optimization as $N+M \to \infty$, as a means of gaining theoretical insights into its generalization capabilities in comparison to classical empirical risk minimization (ERM).
 This analysis assumes a weighting function $\phi(\cdot)$ dependent on both $x$ and $y$. Interestingly, our experiments revealed similar performance between this formulation of $\phi$ and one relying solely on $x$.
 The following proposition shows that asymptotically \ouralgo\ solves a \emph{robust optimization} objective. 
\begin{theorem}[Asymptotics]
\label{asymptotics}
Consider the tri-level optimization in Equation~\eqref{eqn:tri_level}. Suppose the weighting function $\phi(\cdot),$ and splitting function $\Theta(\cdot)$ are dependent on both $x$ and $y$. Let's suppose $N+M\to\infty$, and $\lim_{N,M\to\infty}\frac{M}{N+M} = \delta$. Moreover, suppose the domains of $\phi, \theta, \Theta$ are very large and contain the set of all measurable functions. Then the objective of \ouralgo\  is equivalent to
\begin{equation}
\label{eqn:dual_dro}
    \max_{S: |S|= \delta (N+M) } \min_{\theta} \sum_{(x,y)\in S} \ell(y, f_\theta(x)).
\end{equation}
\end{theorem}
\begin{proof}[Proof Sketch.] The proof of the theorem relies on the observation that there exists a weighting function $\phi$ that can transform any probability distribution $P$ to any another distribution $Q$ (as long as the support of $Q$ is a subset of $P$). Using this observation, one can show that the second, third level optimization problems in Equation~\eqref{eqn:tri_level} is equivalent to following problem: $\min_\theta \sum_{(x,y) \in \Theta(S)^c} \ell(y, f_\theta(x))$.
\end{proof}
A similar result holds when both $\Theta(\cdot)$, $\phi(\cdot)$ rely solely on $x$.  Intuitively, the above theorem shows that in the limit of infinite samples, \ouralgo\ identifies the hardest samples in the training data, and learns a classifier that \textit{minimizes the error on those samples}. This is exactly the goal laid out in our hypothesis above. \vspace{0.02in}

\noindent \textbf{Connections to DRO.} Interestingly, the objective in Equation~\eqref{eqn:dual_dro} is the dual of the following Distributionally Robust Optimization (DRO) objective~\citep{ben2013robust}
\begin{equation}
\label{eqn:dro}
     \min_{\theta} \max_{S: |S|= \delta (N+M) } \sum_{(x,y)\in S} \ell(y, f_\theta(x)).
\end{equation}
DRO is a well studied framework for learning robust models~\citep{namkoong2017variance, duchi2018learning}. However, to the best of our knowledge, dual DRO is not studied in the literature; we will explore this connection more deeply in future work.  

\subsection{Efficient algorithm for validation optimization}
\label{sec:e2e}

Designing a tractable algorithm for meta-optimization runs into two technical challenges: a) How can we learn to assign instances to train and validation sets? and b) How do we efficiently solve the tri-level objective proposed in \cref{sec:hyp}?  We address these two challenges in this section. First, we use a second auxiliary network for soft-assigning instances to train and validation datasets.  Second, we collapse the outer two loops of the trilevel objective into a minimax formulation, and thereby reduce it to a bi-level optimization problem.  We describe each of these in order below.

\noindent \textbf{Soft data assignment:} Inspired by a recent work~\cite{bao2022learning}, we use a partitioning network to predict probability for each instance to be included in the validation set. 
At any point in time, the ``splitter'' network outputs soft assignments $\mathbb{P}(z|x,y)$ with $z\in\{0,1\}$ indicating whether the example $(x,y)$ belongs to the \textit{pseudotest} set ($z=0$) or the \textit{pseudotrain} set. The pseudotrain set is used to train a classifier using standard cross-entropy. The splitter is then updated to identify easy instances for the classifier and assign them to the pseudotrain partition in the next round; this is achieved by minimizing the cross-entropy between its soft-assignment and classifier accuracy:
\begin{equation}
\begin{aligned}
    &\mathcal{L}_{split} = CE(\mathbb{P}_{splitter}(z_i|x_i,y_i), \mathbb{I}_{y_i}(\hat{y}))  \\ 
    & \texttt{where} \quad \hat{y} = \arg \max \mathbb{P}_{predictor}(y|x_i) 
    \end{aligned}
\end{equation}
In order to maintain the label distribution and train-to-test ratio, two regularizers \cite{bao2022learning} are added to penalize shift from the prior distribution, and to push label margins in the training split and testing split close to the original label margin (see supplementary for details).



\noindent \textbf{Meta-optimization with min-max objective:} To simplify the tri-level optimization in \cref{sec:hyp}, 
we  propose a bi-level approximation where the outer loop combines the data splitting and instance reweighting objectives.  Specifically, we propose a min-max game between the splitter (parametrized by $\Theta$) and the meta-network (parametrized by $\phi$), where the splitter has to maximize validation set error whereas instance weights are focused on minimizing it:
\begin{equation}
    \begin{aligned}
        \Theta^*, \phi^* = &\arg \max_\Theta \min_\phi \sum_{(x,y)\in S_t^c}   \left(\ell(y, f_{\theta^*(\phi,\Theta)}(x)) - \mathcal{L}_{split} \right)\\
       where& \quad   \theta^*(\phi,\Theta) = \arg \min_\theta  \sum_{(x,y)\in S_t} \phi(x) \ell(y, f_\theta(x)) \label{eq:bilevel-train}\\
    \end{aligned}
\end{equation}
where the set $S_t = \{(x,y):(x,y)\in S, \mathbb{I}\{\Theta(x,y)>0.5 \}\}$ and $S_t^c = S\setminus{S_t}$, $\mathbb{I}$ denotes the indicator function.
The overall algorithm updates the splitter and instance weighting network once every $Q$ steps of classifier update, and also regularizes the splitter with $\Omega_{ratio}+\Omega_{ld}$ at set intervals (every R steps). The complete description of this method is provided in Algorithm 1 and    analysis of the loss function at the outer level, derivation of update equations for all parameters ($\Theta, \phi, \theta$) are provided in the supplementary. Both the instance weight network and Splitter are parameterized as neural networks. The meta-network for instance weights predict the weight for a training instance ($x_i$, $y_i$) by using ($x_i$) as input, $w_i=g_\phi(x_i)$ and the Splitter predicts a probability of ($x_i$,$y_i$) being in the train set by taking both of them as inputs $z_i=g_\Theta(x_i,y_i)$.
\noindent
In experiments, we refer to this end-to-end optimization method as \textbf{LRWOpt}.

\begin{algorithm}[H]
\caption{LRWOpt: The Overall One-Shot Algorithm.}\label{one_shot_alg}
\begin{algorithmic}[1]
\Require  $\theta$,  $\Theta$,  $\phi$, learning rates ($\beta_1$, $\beta_2$, $\beta_3$), $\mathcal{S}$, $N$, $M$. 
\Ensure Robustly trained classifier parameters $\theta$.
\State Randomly initialize $\theta$, $\Theta$ and $\phi$;
\State initialize ge = 0;
\Comment{Difference b/w train and val error}
\For{e=1 \textbf{to} MaxEpochs}
\State $\mathcal{S}_{tr}, \mathcal{S}_{val} = \text{GenerateSplit}(\mathcal{D},\Theta)$
\For{$b=1$ \textbf{to} M//m} \Comment{m is the batch size}
\State $\{(x^v_i, y^v_i)\}_{i=1}^m$ = SampleMiniBatch($\mathcal{S}_{val},m$);
\State $\Theta \gets \Theta - \beta_1\nabla_{\Theta}\sum  \left (\mathcal{L}_{split} - \ell(y^v_i, f_\theta(x^v_i))  \right)$
\State $\phi \gets \phi - \beta_2\nabla_{\phi}\sum \left (\ell(y^v_i, f_{\theta}(x^v_i)) -\mathcal{L}_{split} \right)$
\For{$j=1$ \textbf{to} $Q$}
\State $\{(x_i, y_i)\}_{i=1}^n$ SampleMiniBatch($\mathcal{D}_{t},n$);
\State $\theta \gets \theta - \beta_3\nabla_{\theta}\sum g_\phi(x_i)\ell(f_\theta(x_i),y_i)$;
\EndFor
\EndFor
\If{$\sum \ell{(y^v_i,f_\theta(x^v_i))}- \sum \ell(y_i, f_\theta(x_i))<$ge}
\State \textbf{break};
\EndIf
\State ge = $\sum \frac{1}{M}\ell(y^v_i,f_\theta(x^v_i))- \frac{1}{N}\sum \ell(y_i, f_\theta(x_i))$
\EndFor
\end{algorithmic}
\end{algorithm}

\begin{figure*}[!htb]
\vskip 0.2in
\begin{center}
\centerline{
    \includegraphics[width=.49\linewidth]{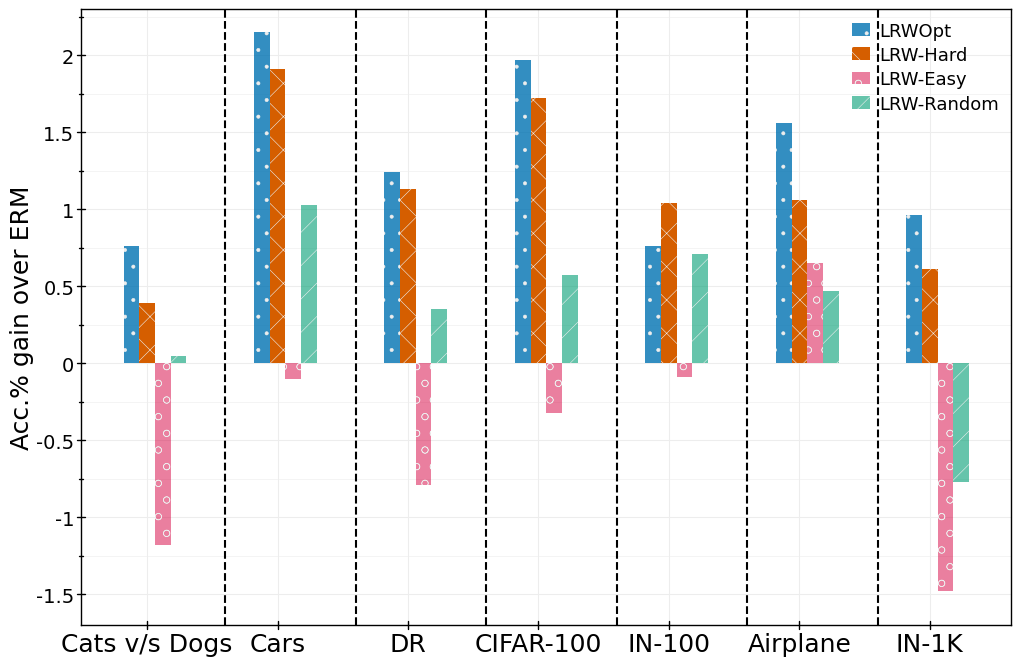}
    \includegraphics[width=.49\linewidth]{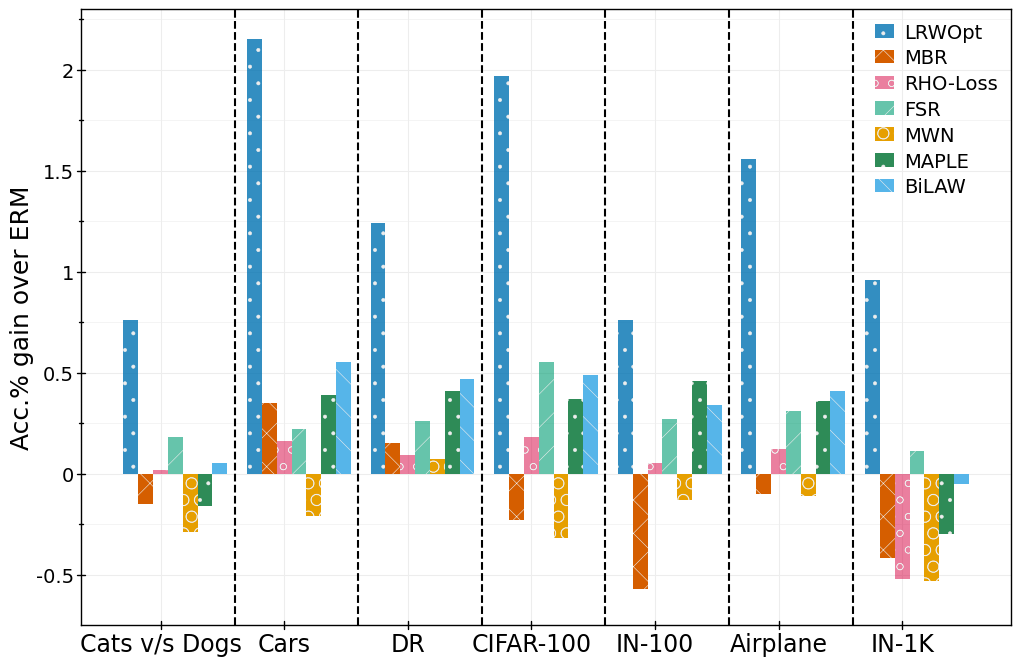}
}
\caption{ \textbf{Robustness analysis on benchmark datasets.}
\textbf{Left:} Comparing different LRW variants, based on the choice of validation set (Easy, Random, Hard, corresponding to the rank-ordering of training data by probabilistic margin of an ERM classifier). $y$-axis shows accuracy gains over ERM for each dataset ($x$-axis). We see consistent ordering of performance, with LRW-Easy $<$ LRW-Random $<$ LRW-Hard, showing the importance of validation set optimization.
\textbf{Right:} Comparing against other re-weighting methods. The figure shows that our proposal (LRW-Hard) outperforms the other reweighting techniques on average, with fast sample re-weighting (FSR) begin competitive in some datasets. In-1K corresponds to ImageNet-1K.
For absolute accuracy values refer supplementary.
}
\label{fig:acc_gain}
\end{center}
\vskip -0.2in
\end{figure*}

\subsection{A simple train-twice heuristic}\label{sec:hard}
We now describe a simple \textit{train-twice} heuristic that can be used to establish the importance of validation set optimization in LRW. We first train an ERM classifier on the available training data, and use the \textit{probabilistic margin} (PM) as a proxy for instance hardness:
\begin{equation}
    PM(x,y,\theta) = \textbf{p}_y(x,\theta) - \max_{j,j\neq y}\textbf{p}_j(x,\theta)
\end{equation}
Here $\theta$ denotes the ERM classifier's parameters. This criterion was used to manually reweight adversarial examples in recent work~\cite{liu2021probabilistic}. Although this proxy score is inexact\footnote{For instance, it measures instance hardness under an ERM classifier, not the to-be-trained LRW model; further, an overtrained classifier may give incorrectly overconfident margins~\cite{liu2021just}.}, it nevertheless allows us to design interesting heuristic LRW variants based on its rank-ordering of training instances: (1) \textbf{LRW-Hard}, where we use the lowest margin instances as validation data, and the rest of the instances as training data, (2) \textbf{LRW-Easy}, in which the \textit{highest margin} instances are used as validation data, and (3) \textbf{LRW-Random}, a control which uses a randomly selected validation set that does not depend on the ERM margin. 

\noindent
We can use these variants to quantify the impact of perturbing the validation set in LRW, and to provide an \textit{existence proof} of validation optimization techniques that materially improve learned classifiers. In particular, we expect to see a clear ordering of classifier test accuracy -- LRW-Easy $<$ LRW-Random $<$ LRW-Hard.

\section{Experiments}

We  perform extensive experimentation on multiple classification tasks including distribution shift benchmarks. For all datasets, if a train-validation split is already available, we use the training data for the ERM classifier in the train-twice heuristic before pooling, ranking, and repartitioning. For the end-to-end optimization, we start with pooled train-validation data and simultaneously learn the data splits and the corresponding LRWOpt model.

\subsection{Datasets} 

We use popular classification benchmarks including CIFAR-100 \cite{krizhevsky2009learning}, ImageNet-100 \cite{tian2020contrastive}, ImageNet-1K \cite{deng2009imagenet}, Aircraft \cite{maji2013fine}, Stanford Cars \cite{krause20133d}, Oxford-IIIT Fine-grained classification (Cats v/s Dogs) \cite{parkhi2012cats} and Diabetic Retinopathy (DR) dataset. Furthermore, for OOD analysis we use the ImageNet-A \cite{hendrycks2021natural}, ImageNet-R \cite{hendrycks2021many} datasets for models trained on ImageNet-1K. We also use the Camelyon \cite{bandi2018detection}, iWildCam \cite{beery2020iwildcam} dataset from the widely popular WILDS benchmark \cite{koh2021wilds} and also the country shifted test set for the DR \cite{kaggle} dataset. We further analyze for robustness in presence of instance dependent noise on the noisy version of CIFAR-10 dataset (Inst.C-10) proposed by \citet{xia2020part} and Clothing-1M \cite{xiao2015learning} dataset.
Please refer supplementary (supp.) for more details, setup for these datasets.

\subsection{Baselines}
Along with ERM, we also compare our method against various re-weighting based methods designed for improving robustness or handling noisy scenarios. These include learned re-weighting methods: MWN \cite{shu2019meta}, FSR \cite{zhang2021learning}, L2R \cite{ren2018learning}, MAPLE \cite{zhou2022model}, BiLAW \cite{holtz2022learning}, GDW \cite{chen2021generalized}, StableNet \cite{zhang2021deep} along with Margin Based Reweighting (MBR) \cite{liu2021probabilistic} and Rho-Loss \cite{mindermann2022prioritized}. Please refer supp. for more details regarding them.

\section{Results}
\subsection{In-Distribution Generalization}



\begin{figure*}[!t]
\vskip 0.2in
\begin{center}
\centerline{
    \includegraphics[width=.43\linewidth, height=0.3\linewidth]{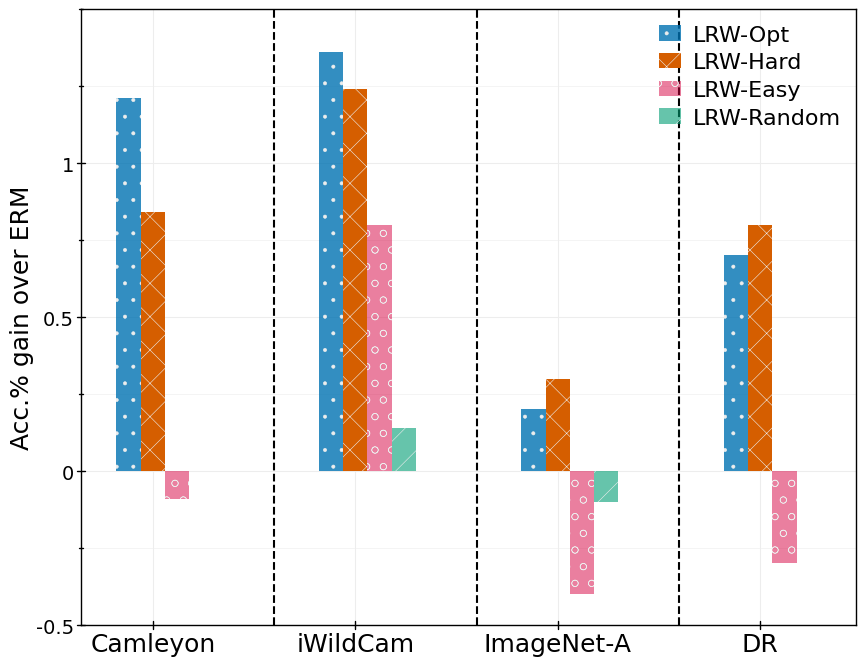}
    \includegraphics[width=.55\linewidth, height=0.3\linewidth]{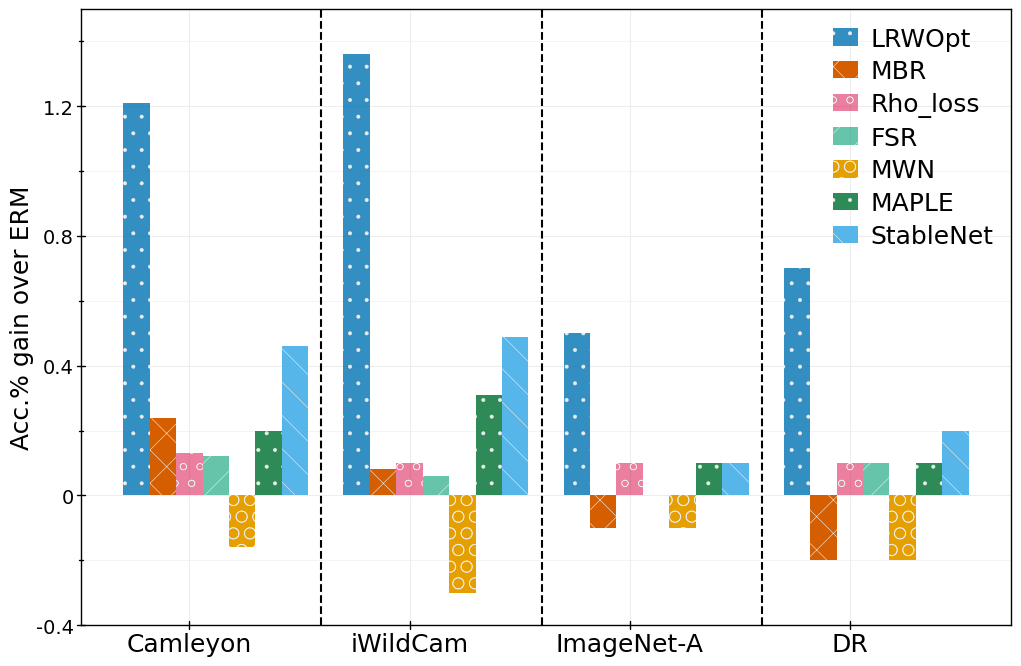}
}
\caption{ \textbf{OOD generalization.}
\textbf{Left:} Comparison of LRW variants on domain shift benchmarks. The ordering between the validation selection methods is reconfirmed on domain shift benchmark datasets as well, suggesting that earlier gains are not via overfitting to training distribution.
\textbf{Right:} Comparing against other re-weighting methods. The figure shows that our proposal (LRW-Hard) outperforms the other reweighting techniques on average, with fast sample re-weighting (FSR) begin competitive in some datasets. For absolute values refer supp.
}
\label{fig:domainshift}
\end{center}
\vskip -0.2in
\end{figure*}



\subsubsection{\ouralgo\ improves classification accuracy}

Figure~\ref{fig:acc_gain} (left) shows the LRW advantage (gains over an ERM baseline) for a range of datasets. Shown are 4 alternatives--LRW-Easy, LRW-Random, and LRW-Hard, corresponding to easy, random, and hard validation sets per the train-twice heuristic (\cref{sec:hard}), and LRWOpt, which is the end-to-end optimization of LRW classifier and train-validation split (\cref{sec:e2e}). Refer supplementary for absolute accuracy values.

\begin{itemize}
\item Gains vs ERM are strictly ordered in nearly all datasets--LRW-Easy $<$ LRW-Random $<$ LRW-Hard--robustly confirming our hypothesis that validation sets in LRW classifiers need to be optimized. The difference between LRW-Hard and LRW-Easy is over 2.5\% relative gain for Imagenet \& CIFAR-100, and 1.6\% relative for Clothing-1M.
\item LRW-Random shows modest gains over ERM in most datasets\footnote{Previous results~\cite{ren2018learning} had shown neutral or slightly negative gains over ERM, focusing chiefly on training with noisy training labels + clean validation data. The improvement in our results are driven by the use of the meta-network.}. LRW-Hard shows significant gains, underscoring the benefit of combining LRW with validation set optimization --  0.8\% relative gain on Imagenet,  1.3\% on Clothing-1M, 2.15\% relative on CIFAR-100. 
\item The end-to-end LRWOpt matches or exceeds LRW-Hard, showcasing its effectiveness without the need for training twice. On a minority of datasets, LRW-Hard is nominally better; we believe this is due to the small-sample nature of real-world datasets, and attendant estimation noise.
\end{itemize}

Although the presented accuracy is on unseen test sets, a concern may be that LRW-Hard overfits to the training data distribution and results in brittle classifiers. To address this, we performed a number of experiments on datasets with matched out-of-domain test sets (Diabetic retinopathy~\cite{kaggle,aptos}, Camelyon~\cite{bandi2018detection}, WildCam~\cite{beery2020iwildcam}) in next subsection.

\subsubsection{\ouralgo\ outperforms existing re-weighting baselines}

We  now compare our LRWOpt method against other reweighting methods-- FSR~\cite{zhang2021learning}) , MBR~\cite{liu2021probabilistic}, MAPLE \cite{zhou2022model}, MWN \cite{shu2019meta}, StableNet \cite{zhang2021deep}, BiLAW \cite{holtz2022learning})--see \cref{fig:acc_gain} (right) for more details. While FSR, RHO-Loss, MBR, MAPLE, StableNet and BiLAW are proposals in the literature for learned reweighting of the training data (via meta-learning) in order to address various clean- and noisy-label scenarios, MBR is an extrapolation of a proposal for ad-hoc upweighting of \textit{adversarial examples} alone~\cite{liu2021probabilistic}, for increased robustness against those specific adversarial attacks. In similar spirit, the proposal ``just train twice''~\cite{liu2021just} also suggests an ad-hoc upweighting of poor performing instances in a second round of training. RHO-loss~\cite{mindermann2022prioritized}, on the other hand, follows the harder version as against all these methods and selects only some points from a batch which minimize the holdout set loss for efficient training. \\
On the datasets tested, we see that LRWOpt clearly outperforms all these baseline methods consistently across all the datasets. It is followed by MAPLE, StableNet which are designed keeping group/OOD robustness in mind. FSR and Rho-Loss perform similar to ERM only. 
\noindent
Interestingly, MBR and MWN are among the worst performers, showing that it is insufficient to simply upweight instances based on a first-round estimate of margin and the commonly followed bi-level optimization procedure by using free-parameters/loss based re-weighting, even though capable of providing unbiasedness, is not sufficient and requires another level of optimization for finding the appropriate validation set. 


\subsection{Out-Of-Distribution (OOD) generalization}\label{sec:oodval}

We further evaluate the robustness of our learned representations on the standard out-of-distribution datasets from the Wilds benchmark namely the camelyon \cite{bandi2018detection} and the iwildcam \cite{beery2020iwildcam} datasets along with the Diabetic Retinopathy dataset with a country shifted test set (APTOS test dataset) \cite{kaggle}. For the first two, 10\% examples from the train set, sampled randomly, are used as val set. For the last one, a separate validation set, with same domain as the train set, is provided. Refer supplementary for more details regarding the dataset details and our evaluation setup.
Furthermore, we also analyze the performance on the ImageNet-A dataset by using both training and validation data from ImageNet-1K. 
 Our primary goal here is to check whether the gains we saw above are primarily driven by in-domain learning, or a broader improvement in generalization capacity over the existing learned reweighting methods
 --as a result, we restrict ourselves to comparison among ERM, the LRW variants along with LRWOpt and the re-weighting baselines, rather than the substantial literature on domain shift. \cref{fig:domainshift} (left) confirms that our learned classifiers generalize better to domain shift data as compared to ERM classifiers; further, the ordering between the different validation datasets is largely preserved.
\cref{fig:domainshift} (right) provides evaluation of our methods and other re-weighting baselines. These include the standard bi-level optimization based ones like MWN, FSR along with ones designed for OOD generalization/robustness like MAPLE, StableNet, Rho-Loss and MBR. It can be observed that our method surpasses all of these baselines thereby showing the importance of validation set optimization problem for the bi-level optimization based re-weighting methods, introduced in the paper. StableNet comes out to be the closest competitor of our method which has been designed explicitly for OOD learning. Rest all the re-weighting methods perform similar to ERM with minor gains/losses. Absolute accuracy values for this analysis are provided in the supplementary.
\\


\subsection{Practical Label Noise Settings}
We now study the noisy label setting with a focus on real-world label noise using datasets like Clothing1M \cite{xiao2015learning}, and noisy CIFAR-10  (Inst. C-10~\cite{xia2020part}), which contain instance conditioned noise as against randomly flipping labels or adding uniform noise. This noise can be treated as a measure of instance hardness rather than labeling error (refer supplementary for more details). Table \ref{tab:int_noise} compares LRWOpt method against bi-level optimization approaches like MWN, FSR, L2R, GDW which although specifically designed for noisy labels have been primarily tested on label flipping, uniform noise data. We also compare against  MAPLE, designed for general robustness. On Inst. C-10, our method clearly surpasses all baselines, with 1.53\% accuracy gain. On Clothing-1M, we show about 0.85\% gain over the best performing GDW baseline. For other noisy settings like random label flips, prior work on outlier robust DRO \cite{zhai2021doro} has shown good results using loss clipping (e.g. exclude k\% of highest loss as label noise), which could be co-opted into our work.

\subsection{Skewed Labels}
We now study the skewed label setting, comparing against instance based reweighting schemes proposed with this setting in mind (Meta-Weight-Net, FSR and another recently proposed re-weighting method GDW \cite{chen2021generalized}). Table \ref{tab:skewed} shows the results for this analysis on the CIFAR-100 datasets with various skew levels ranging from 1 to 200. It can be observed that our LRWOpt significantly outperform the existing loss based reweighting methods at all of the skew levels showing accuracy gains upto 2.22\%. 

\subsection{\ouralgo\ scales to large pretrained models}
We further analyze the LRW-Hard, Easy and Random methods along with the ERM baseline on ViT-B/16 pretrained backbone trained using these techniques and evaluated on the ImageNet-1K dataset. Table \ref{tab:pretrained_imnet} shows the results from this experiments. Again LRW-Hard emerges the winner and improves significantly (around 1\%) w.r.t. other methods and the baseline advocating robustness using hard examples. Furthermore, there is 1.92\% difference in accuracy between LRW-Easy and LRW-Hard, even though we warm-started all the techniques with a pretrained backbone, showing the sensitivity of techniques to choice of validation set.


\begin{table}[]
    \centering
    \footnotesize
   \begin{tabular}{c|ccccccc}
    \toprule
        Method & Easy & Hard & Random & LRWopt & MWN  & ERM \\
        \midrule
        Acc. & 83.17 & \textbf{85.09} & 83.87 & 84.94 
 & 84.02  & 84.11 \\
        \bottomrule
    \end{tabular}
    \caption{ImageNet-1k dataset with a ViT-B/16 pretrained backbone. We compare various versions of our method with the ERM, Meta-Weight and l2s baselines.}
    \label{tab:pretrained_imnet}
\end{table}

\begin{table}[]
    \centering
    \footnotesize
    \begin{tabular}{c|ccccc}
    \toprule
        class skew & 200 & 50 & 10 &1 \\
        \midrule
        MWN \cite{shu2019meta}	& 40.11 $\pm$ 0.9 &	48.67 $\pm$ 0.7&	61.32 $\pm$ 0.6 &	74.23 $\pm$ 0.3\\
        FSR \cite{zhang2021learning}	& 38.04 $\pm$ 0.8	& 45.12	$\pm$ 0.9 & 58.38	$\pm$ 0.6 & 74.68 $\pm$ 0.2 \\
        GDW \cite{chen2021generalized} & 40.36 $\pm$ 1.0  & 48.89 $\pm$ 0.8 & 61.67 $\pm$ 0.5 & 74.41 $\pm$ 0.4 \\
        \midrule
        LRWOpt &	\textbf{42.33} $\pm$ 0.8 &	\textbf{50.77} $\pm$ 0.7 &	\textbf{63.28} $\pm$ 0.8 &
        \textbf{75.12} $\pm$ 0.3 \\
        \bottomrule
    \end{tabular}  
    \caption{Comparison, on CIFAR-100 dataset, of our LRWOpt method with existing meta-learning based reweighting methods in a label skew setup for which these methods were defined.}
    \label{tab:skewed}
\end{table}

\begin{table}[]
    \centering
    \footnotesize
   \begin{tabular}{c|ccccccc}
    \toprule
         & MWN & FSR & L2R & MAPLE & GDW  & Ours \\
        \midrule
        Inst. C-10 & 65.89& 67.12 & 70.21 & 70.34 
 & 69.12  & \textbf{71.87} \\
 Clothing-1M & 72.79 & 72.07 & {72.22} & 71.67 & 73.12 & \textbf{73.97} 
 \\
        \bottomrule
    \end{tabular}
    \caption{\textbf{Instance-dependent noise}. ImageNet-1k dataset with a ViT-B/16 pretrained backbone. We compare various versions of our method with the ERM, Meta-Weight and l2s baselines.}
    \label{tab:int_noise}
\end{table}
    


\begin{figure*}[!htb]
\vskip 0.2in
\begin{center}
\centerline{
    \includegraphics[width=0.25\textwidth]{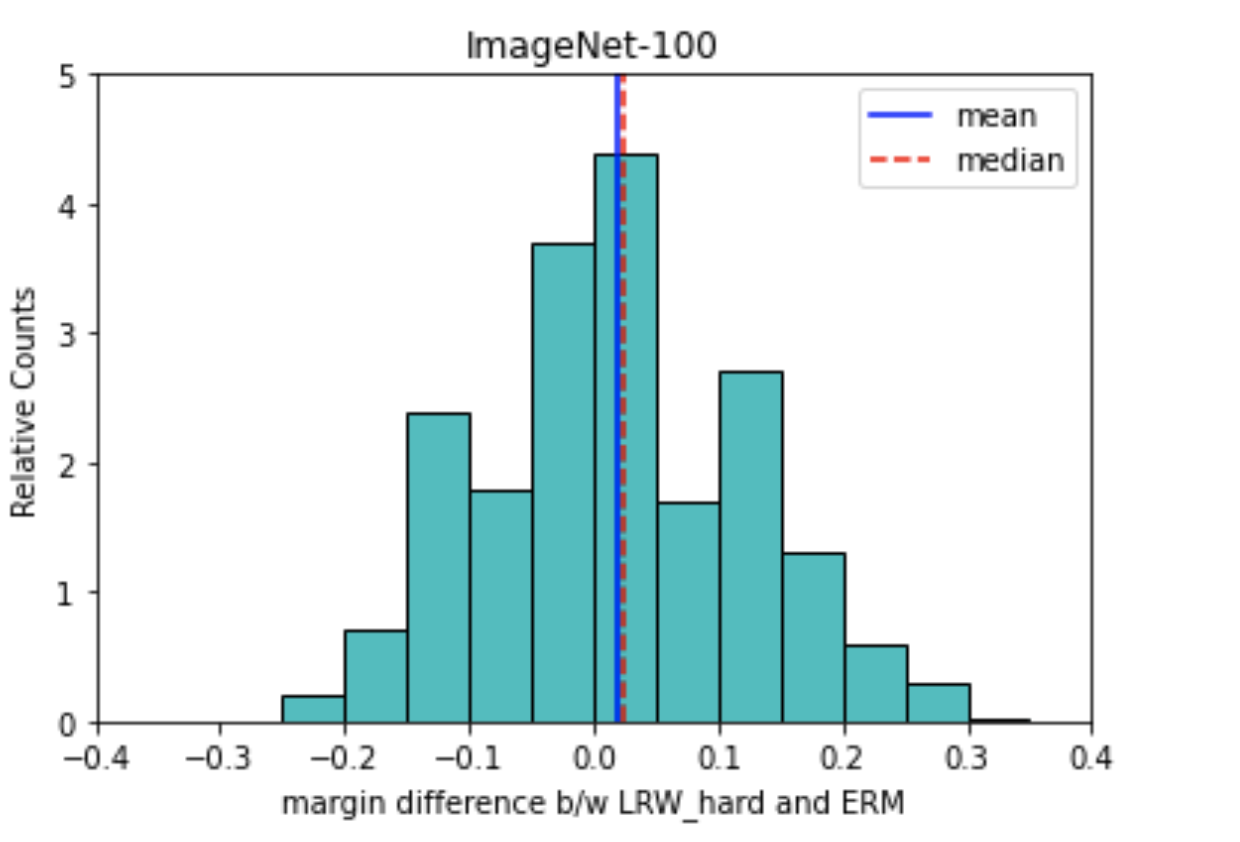}
    \includegraphics[width=.25\linewidth]{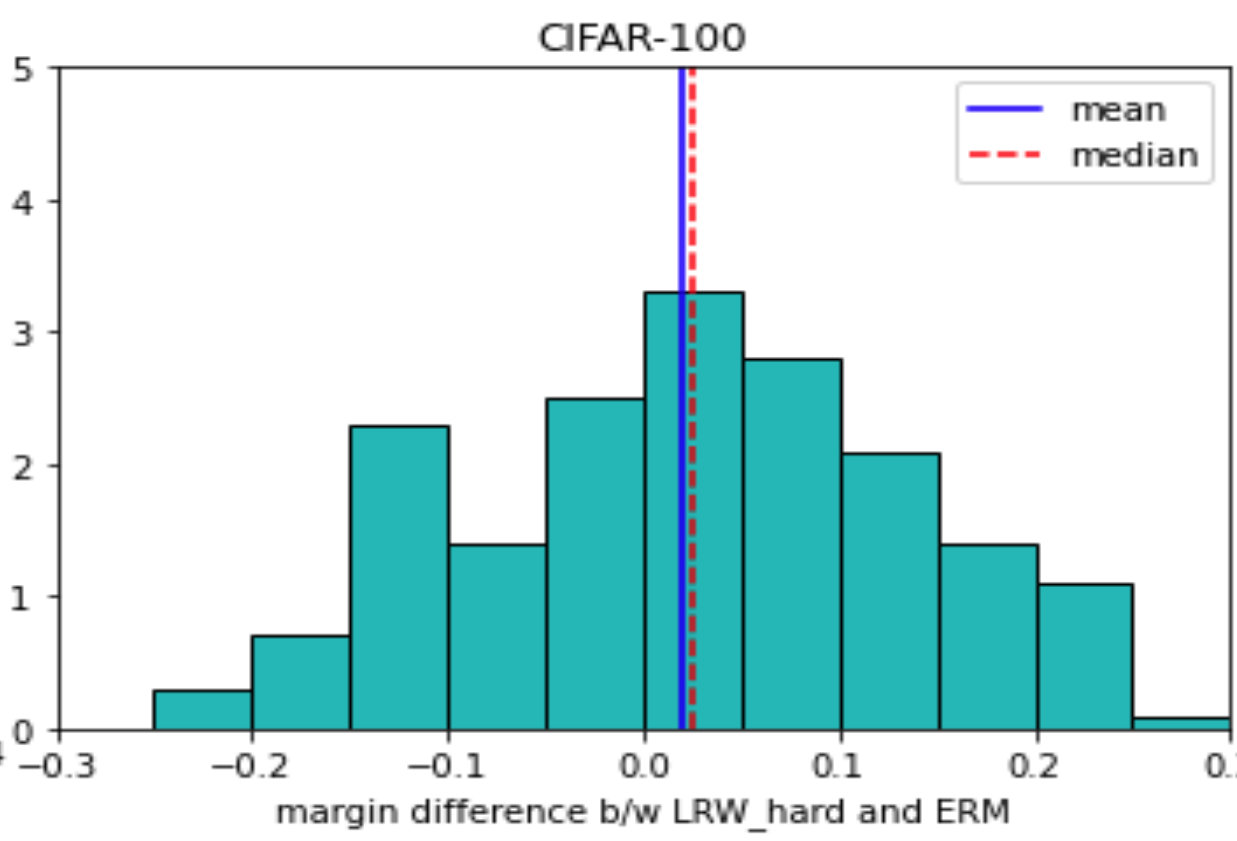}
    \includegraphics[width=0.255\textwidth]{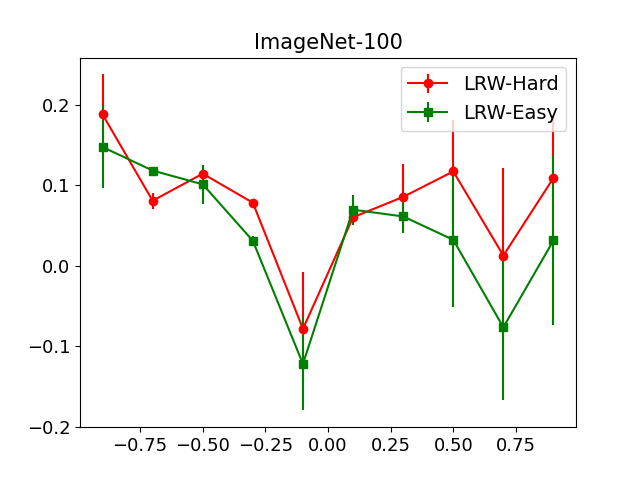}
    \includegraphics[width=0.23\textwidth]{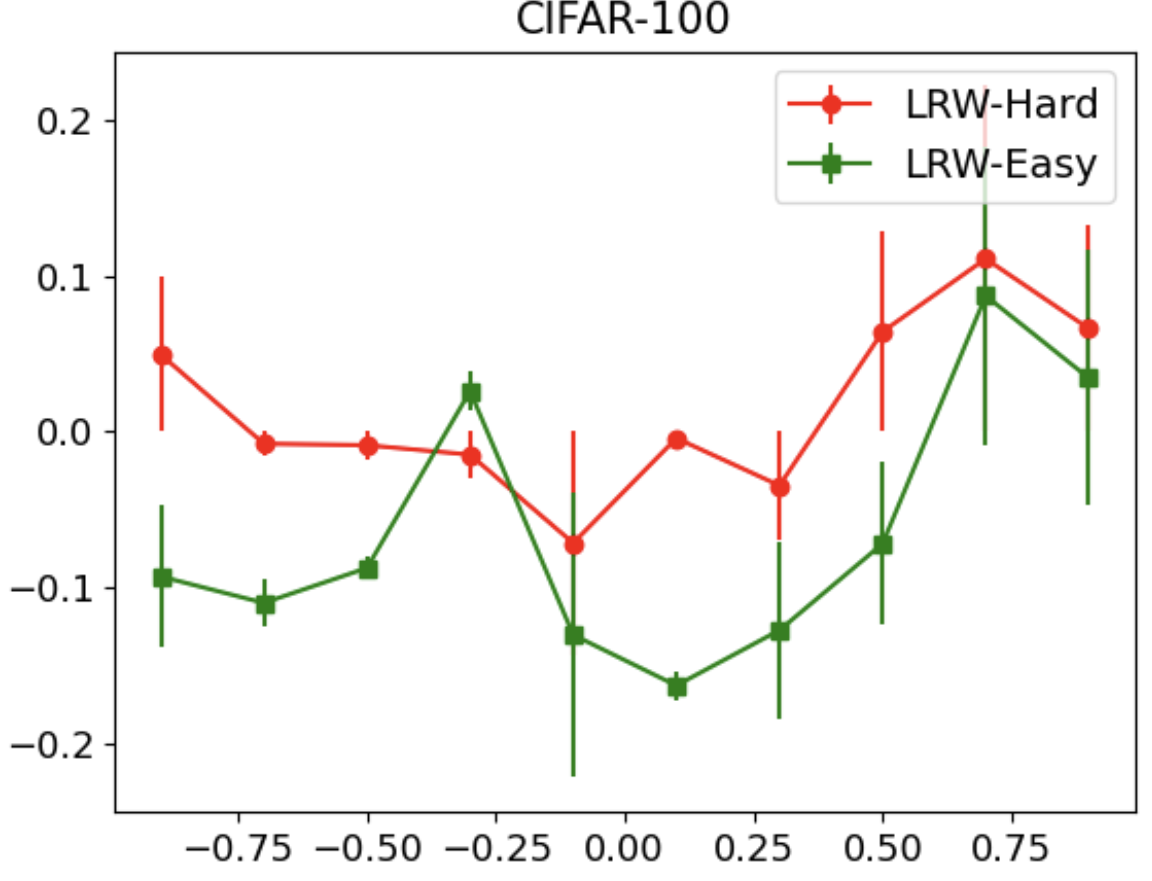}
}
\caption{\ouralgo\ improves margins of learned classifiers. (a,b): paired margin deltas between LRWOpt and ERM are moderately right-skewed with mean/median greater than zero. (c,d): As a function of ERM margin, clear separation seen between LRW-Hard (better) and LRW-Easy (worse) in terms of margin gain over ERM (errorbars are SEM). All results on unseen test data; Imagenet-100, CIFAR100 shown for brevity, with similar results for other datasets in supplementary. }
\label{fig:margin_delta}
\end{center}
\vskip -0.2in
\end{figure*}

\subsection{Leveraging OOD val set: A heuristic solution}

We now turn to the use of known hard instance, or OOD, datasets when available, in the \ouralgo\ context. For the well-studied Imagenet-1K dataset, two datasets are frequently used to gauge generalization properties of learned classifiers -- IN-A (``natural adversarial instances'', examples known to be misclassified by a trained Resnet-50 model), and IN-R (renditions of objects, such as drawings, paintings, sculptures, etc).  For each of these, we used a portion of the dataset as LRW validation set, and retained another portion for testing.  Table~\ref{tab:naturalood} shows the results of this study. We report the following two surprising results: (1)  our approach outperforms the baseline ERM classifier on not just the OOD test set, but \textit{also on the in-domain test set} (1.7-3\% and 0.5-0.7\% gains respectively), and (2) \textit{simply adding the OOD data to the train set} only moderately improves OOD test set performance while \textit{degrading} ID accuracy. These findings were replicated for both IN-A and IN-R as LRW validation sets, underscoring the real-world applicability of our findings, on large-scale datasets.\\

\begin{table}[!thb]
    \centering
    \footnotesize
    
   \begin{tabular}{c|cc|cc}
    \toprule 
    & \multicolumn{2}{c|}{IN-R val (ResNet50)} & \multicolumn{2}{c}{IN-A val (ResNet152)} \\
       	& IN1K Test	 & INR Test & IN1k Test & INA Test\\
       \midrule
        LRWOpt + INR val & 	\textbf{76.14}	& \textbf{49.1}&\textbf{ 78.12}	& \textbf{7.9}\\
ERM (IN1K)	& 75.65	& 46.1& 77.31	& 6.2\\ 
ERM (IN1K+INR)	& 74.89	& 47.4 & 77.08	& 6.6   \\
\bottomrule
   \end{tabular}


\caption{Natural hard examples as validation. LRW classifiers on Imagenet (IN1k) data and OOD validation sets (IN-A \& IN-R respectively) outperform ERM on \textit{both in-domain and OOD} test sets. Simply augmenting training data of ERM baseline does not match this (see text for details).}
 \label{tab:naturalood}
\end{table}



\vspace{-2.0mm}
\subsection{Margin maximization via meta-learning}



    

We present empirical evidence showing that \ouralgo\ has a \textit{margin maximization effect}, i.e., the learned classifiers have wider margins than ERM classifiers. Since our validation data is selected to be low-margin instances, our LRW classifier upweights, and improves performance on, training instances most similar to the low-margin validation set.  \ref{fig:margin_delta} shows two views of margin differences between \ouralgo\ and ERM on the test set, to confirm this expectation\footnote{We show Imagenet-100 \& CIFAR100 for brevity; findings consistent across all datasets (see supplementary).}. Panels (a,b) shows a histogram of paired margin differences between ERM and LRWOpt, indicating a modest right-skew with mean \& median to the right of 0. Panels (c,d) shows LRW-Hard and LRW-Easy deltas w.r.t. ERM, averaged over ERM margin buckets (mean and S.E.M. errorbars)--across the board (i.e., for most values of ERM margin), LRW-Hard contrasts with ERM better than LRW-Easy.


\noindent
Further details on Algorithm 1, the time complexity of the proposed scheme compared to ERM, the training setup, ablation study over the validation set selection/loss, and discussion regarding the early-stage performance of the proposed scheme are provided in the supplementary.

 \vspace{-0.1in}
\section{Discussion \& conclusion}

We proposed the novel idea of \textit{optimizing the choice of validation data} in a learned-reweighting setting, and showed that it gives significant gains over ERM on a range of datasets and domain generalization benchmarks. In particular, in most experiments on clean data, we saw a clear ordering between choosing easy, random, and hard samples as validation data in an LRW setup, with the latter performing best and delivering consistent gains over ERM. This ordering provides broad support to our primary hypothesis: meta-optimization of the metalearning workflow in LRW is an important area of research with potential for substantial impact. Our specific heuristic of choosing low-margin points is a simple, straightforward instantiation of what we believe is a family of optimization algorithms that can be brought to bear on the general problem of optimizing meta-learned classifiers. Indeed, our heuristic is not competitive under very high label noise scenarios, suggesting the need for follow-on work that explores more formal, optimization-driven approaches towards this problem. We are also excited about elucidating the theoretical basis of observed gains in the \ouralgo\ framework.

{
    \small
    \bibliographystyle{ieeenat_fullname}
    \bibliography{main}
}

\twocolumn[\section*{\centering \Large Supplementary Material}]

\appendix
\section{Algorithmic Description}

 Algorithm 1 in the main paper covers the details of the one-shot LRWOpt scheme. It requires a set of initialized Splitter parameters $\phi$, Meta-Network parameters $\Theta$ and classifier parameters $\Theta$ and finally outputs a set of optimal classifier parameters $\theta^*$. 
Also, the splitting of the total dataset ($\mathcal{D}$) into train ($\mathcal{S}^{tr}$) and validation set $\mathcal{S}^{val}$ is done using the Splitter parameterized as a neural network $f_\Theta$ ($\Theta$ being the parameters) and $0<f_\Theta(x,y)<1$ for any instance (x,y). 
The examples with $f_\Theta(x,y)>0.5$ are put into validation set. $F_\Theta$ denotes the application of this splitting function onto the overall dataset outputting  train and validation subsets, \textit{i.e.} $F_\Theta(\mathcal{D}) = \{x_i,y_i:(x_i,y_i)\in \mathcal{D}; f_\Theta(x_i,y_i) \geq 0.5\}, \{x_i,y_i:(x_i,y_i)\in \mathcal{D}; f_\Theta(x_i,y_i)<0.5\}$. Here, instead of applying the nested loops for the bi-level setup at the epoch level, we have done it at the batch level. Both of them yield nearly similar results.

\section{Proof of Theorem~$1$}
\begin{reptheorem}{asymptotics}[Asymptotics]
Consider the tri-level optimization in Equation~$(2)$. Suppose the weighting function $\phi(\cdot),$ and splitting function $\Theta(\cdot)$ are dependent on both $x$ and $y$. Let's suppose $N+M\to\infty$, and $\lim_{N,M\to\infty}\frac{M}{N+M} = \delta$. Moreover, suppose the domains of $\phi, \theta, \Theta$ are very large and contain the set of all measurable functions. Then the objective of \ouralgo\  is equivalent to
\begin{equation}
\label{eqn:dual_dro_rep}
    \max_{S': |S'|= \delta (N+M) } \min_{\theta} \sum_{(x,y)\in S'} \ell(y, f_\theta(x)).
\end{equation}
intuitively, the model picks points close to boundary into validation set. 
\end{reptheorem}
\begin{proof}
    Let $S'$ and $S\setminus S'$ be any partitioning of the dataset $S$. Let $Q^{\text{val}}, Q^{\text{tr}}$ be the probability distributions corresponding to $S'$, and $S\setminus S'$. The proof proceeds by showing that the following two optimization problems are equivalent
    \begin{equation}
    \begin{aligned}
    \label{eqn:bilevel_core}
        \min_{\theta} \mathbb{E}_{(x,y)\sim Q^{\text{val}}}[\ell(y, f_\theta(x))].
   \end{aligned}
    \end{equation}
    \begin{equation}
\begin{aligned}
\label{eqn:trilevel_core}
&\min_\phi \mathbb{E}_{(x,y)\sim Q^{\text{val}}} [\ell(y, f_{\theta^*(\phi)}(x))]\\
  \texttt{s.t.}   \theta^*(\phi) &= \arg \min_\theta \mathbb{E}_{(x,y)\sim Q^{\text{tr}}} [\phi(x,y) \ell(y, f_\theta(x))].
\end{aligned}
\end{equation}
Observe that the above two optimization problems are the inner optimization problems of objectives~\eqref{eqn:tri_level} and~\eqref{eqn:dual_dro_rep}.
Showing that these two are equivalent would then immediately imply that objective~\eqref{eqn:tri_level} is equivalent to objective~\eqref{eqn:dual_dro_rep}. 

 First let's consider the case where $\text{supp}(Q^{\text{tr}}) = \text{supp}(Q^{\text{val}})$. By choosing $\phi(x,y) = \frac{Q^{\text{val}}(x,y)}{Q^{\text{tr}}(x,y)}$, the constraint in Equation~\eqref{eqn:trilevel_core} can be rewritten as
\[
\theta^*(\phi) = \arg \min_\theta \mathbb{E}_{(x,y)\sim Q^{\text{val}}} [\ell(y, f_\theta(x))].
\]
Observe that this is the same as the optimization problem in Equation~\eqref{eqn:bilevel_core}. This shows that both the optimization problems in Equation~\eqref{eqn:bilevel_core},~\eqref{eqn:trilevel_core} are equivalent. 

Next, consider the case where $\text{supp}(Q^{\text{tr}}) \neq \text{supp}(Q^{\text{val}})$. This is the easy case to handle. To see this, consider the extreme case where  $\text{supp}(Q^{\text{tr}}) \cap \text{supp}(Q^{\text{val}}) = \{\}$. Since the domain of $\theta$ contains the set of all measurable functions, we can set $\phi(x,y) = 1$ and choose $\theta^*(\phi)$ to be the Bayes optimal classifier\footnote{A Bayes optimal classifier is a classifier that minimizes the expected population risk} on both $\text{supp}(Q^{\text{tr}}), \text{supp}(Q^{\text{val}})$. It is easy to verify that this is an optimizer of Equation~\eqref{eqn:trilevel_core}. This shows that Equation~\eqref{eqn:bilevel_core},~\eqref{eqn:trilevel_core} are equivalent. A similar argument can be used to handle the more general case of $\text{supp}(Q^{\text{tr}}) \neq \text{supp}(Q^{\text{val}})$. Here, we choose a $\phi$ that performs probability matching on the intersection of the two supports, and set $\phi(x,y)=1$ on the rest of the support.
\end{proof}

\section{Deriving the Update Equations}
Let us now discuss the update equation for each of the neural networks namely the Splitter Network ($\Theta$), the Meta-Network ($\phi$) and the target Prediction Network ($\theta$). As discussed in the paper and in the algorithmic description provided above, we have formulated the problem as a bi-level optimization task with $\Theta$, $\phi$ being optimized at the outer level and $\theta$ at the inner level. Therefore, for every update in $\Theta$, $\phi$, we update $\theta$ for $K$ steps as an approximation for most optimal $\theta$ for the current value of $\Theta$ and $\phi$, \textit{i.e.} $\theta^*(\phi,\Theta)$.

\noindent
\textbf{Splitter Network ($\Theta$)}. Updated at the outer loop level using the validation set to minimize the loss to identify whether a given input label pair would be predicted correctly and to maximize the loss on the validation set when using the current classifier parameters, for maximum generalization error. After $e$ epochs of the complete bi-level setup, its update equation can be written as:

\begin{equation}
\begin{aligned}
    \Theta_{e+1} = \Theta_e - \frac{\beta_1}{M}\sum_{i=1}^{M} &\nabla ({CE}(\mathbb{P}_{splitter}(z^{v}_{i}|x^{v}_{i},y^{v}_{i}),\mathbb{I}_{y^{v}_{i}}(\hat{y}^{v}_{i}))\\
    &-
    l_{val}(y^{v}_{i}, f_{\theta^{*}(\Theta,\phi)}(x^{v}_{i}))\\
    \end{aligned}
\end{equation}
As discussed in the paper, along side this loss, the regularizers proposed in \cite{bao2022learning} are also used to update the splitter.
\noindent
\textbf{Meta-Network ($\phi$)}. This is always updated alongside the splitter objective where the loss term corresponds to minimizing the error on validation set. Thus, after $e$ epochs, it can be written as: 
\begin{equation}
\begin{aligned}
\phi_e = \phi_e - \beta_2\nabla_{\phi}&\sum_{i=1}^{M} (l_{val}(y^v_i, f_{\theta^*(\Theta,\phi)}(x^v_i)) \\-{CE}&(\mathbb{P}_{splitter}(z^v_i|x^v_i,y^v_i),\mathbb{I}_{y^v_i}(\hat{y}^v_i))  )  
\end{aligned}
\end{equation}

The overall setup is based on Meta-Network being independent of the Splitter and only aimed at making classifier generalize well on the validation set and thus, the second term inside the gradient can be asssumed independent of $\phi$ leading to:
\begin{equation}
\phi_e = \phi_e - \beta_2\sum_{i=1}^{M}\frac{\partial }{\partial \phi}l_{val}(y^v_i, f_{\theta^*(\Theta,\phi)}(x^v_i))
\label{meta_final}
\end{equation}
\noindent
\textbf{Classifier Network ($\theta$)}. Given the algorithm, after $e$ epochs of the complete bi-level setup, it would have led to $Ke^{'}$ epochs over the training data, of the classifier where $e^{'}$ corresponds to number of epochs on train data while one epoch on val data is completed and $K$ is the number of times inner loop is run for every outer loop. We have $e^{'} = \frac{batch\_t}{batch\_v}e$. The classifier has to be just updated through the weighted training loss on the split $\mathcal{D}^t$:

\begin{equation}
    \theta_{Ke^{'}+1} = \theta_{Ke^{'}} - \beta_3\sum_{j=1}^{N}\nabla_\theta g_\phi(x_i)l(f_\theta(x_i),y_i)
\end{equation}

Following the recent reweighting works\cite{jain2024learning, shu2019meta}, we also approximate it as:
\begin{equation}
    \theta_{Ke^{'}+1} = \theta_{Ke^{'}} - \beta_3\sum_{j=1}^{N}g_\phi(x_i)\nabla_\theta l(f_\theta(x_i),y_i)
    \label{classifier_final}
\end{equation}

The Classifier and Meta-Network update equations (eqs. \ref{classifier_final} and \ref{meta_final}) are same as the existing instance based re-weighting works \cite{jain2024instance, jain2024learning} having a validation set bi-level set which they approximate as a single level  optimization set.
This involves creating a copy of the classifier ($\hat{\theta}$) and using that to update the meta-network ($\phi$). 
Thus, $\phi$
For more details and derivations for the update equations of classifier and Meta-Network, following works \cite{jain2024instance, jain2024learning} can be referred.

\noindent
\textbf{Early Stage Performance and Convergence:} Initially, the splitter network is likely to randomly assign data to training and validation, and the scorer network will assign random weights -- in expectation, we believe this will fall back to the baseline ERM performance during initial training epochs. Empirically, examining learning curves of LRWOpt vs ERM, we see similarity in early epochs followed by gradual divergence. Also, previous  work provides convergence guarantees for bi-level \cite{xiao2023generalized, shu2019meta} and min-max \cite{chen2022accelerated} objectives using alternating updates for learning.

\section{Experimental Details}
\subsection{Training and Evaluation}
\textbf{Architectures.} We have used following architectures for classifer: WRN28-10 for CIFAR-100, VGG-16 for ImageNet-100, ResNet-152 for Oxford-IIIT dataset,  ResNet-32 for the aircraft and stanford cars datasets, 
and ResNet-50 for the rest. We add dropout regularization, following \cite{jain2024learning}, to the classifier.
We used a pretrained backbone as the base of the meta network having the same architecture as the classifier, to which we attach a fully connected layer for predicting the instance weights.
For Splitter we follow the architecture from the learning-to-split\cite{holtz2022learning} paper again the classifier backbone in a read-only manner followed by a learnable MLP layer to predict the splitting decision.  \\
\textbf{Training.} For training the classifier, we use a batch size of 64 and image size of $224\times 224$ for all experiments except CIFAR-100 where it is $32\times 32$. We use an initial learning rate of $0.1$ for the classifier, followed by a factor 10 decay every 50 epochs. For the meta network and splitter, we fix the learning rate at $1e-3$. We use a momentum value of 0.9 for all three. We run each experiment for 100 epochs of training, for which we observed convergence in all our experiments.
 We warm-start the main classifier by training for 25 epochs on a random split, followed by updating the meta-network and splitter for every 5 updates to the classifier ($Q$=5). We used a dropout rate of 0.25 for the classifier network, with 5 evaluations for estimating variance. The split of training/validation data, for different datasets, into the train and meta train sets is provided below along with dataset descriptions. We keep the length of training set same for all the methods and the baselines and do hyperparameter tuning on the validation set for the methods not using it in their optimization. We fix delta to be 0.1, resulting in a validation set of at most 10\% of the training set;  based on our splitting criterion, this target is always reached. In our current method, even if delta is set higher, only those examples will be included for which $\Theta(x,y)<0.5$, thereby enforcing a hardness constraint.  Also, both from experiments and DRO literature, we discovered increasing delta decreases variance but increases bias. 

\noindent
\textbf{Regularizer details}. We used the following two regularizers \cite{bao2022learning}:
\vspace{-0.1in}
\begin{equation*}
    \Omega_1 = D_{KL}(\mathbb{P}(z|B(\delta))), \Omega_2 = \sum_{k\in \{0,1\}}D_{KL}(\mathbb{P}(y|z=k))
\end{equation*}
where $\mathbb{P}(z)$ denotes the percentage of examples in train or val set and $B(\delta)$ is Bernoulli($\delta$). The first regularizer guides the splitter to maintain train/validation ratio close to $\delta$ and the second aims to balance labels across splits. We included them based on prior work; however, their contribution was modest -- roughly 0.23\% accuracy across datasets, with the second regularizer contributing most of it ($\sim$ 0.2\%). 


\subsection{Datasets} 
As discussed in the paper, we have used popular classification benchmarks  CIFAR-100, ImageNet-100, ImageNet-1K, Aircraft, Stanford Cars, Oxford-IIIT Fine-grained classification (Cats v/s Dogs) and Clothing-1M. Alongside these we have used ImageNet-A, ImageNet-R, Camelyon, iWildCam and Diabetric Retiopathy dataset with a country shift setup for OOD analysis using ImageNet-1K trained models.\\
\textbf{ImageNet-1K} \cite{deng2009imagenet} consisting of 1.3M images across 1000 classes, is the largest of the datasets in our experiments. We have used tha training set as the overall dataset (train+val) for applying our setup and training baselines. \\
\textbf{ImageNet-100} \cite{tian2020contrastive}. A subset of ImageNet-1K with 100 classes, 130k training instances and 5k examples for testing as a validation set. We use 13k examples from the train set as our validation set, and the rest for training.\\
\textbf{Clothing-1M} \cite{xiao2015learning}. Around 1M images from 14 apparel classes; since images and labels are programmatically extracted from the web, there is significant label noise.  Around 72k manually refined examples form a clean subset, of which 10k examples comprise the test set and the remaining 11k are marked as validation. \\
\textbf{CIFAR-100} \cite{krizhevsky2009learning}. This dataset consists of 60k images of size $32\times 32$ spread across 100 categories. We use 50k images for training and 10k for testing. For the meta-network, we use 5k examples from the train set as validation data, and the remaining 45k examples for training the classifier. \\
\textbf{Inst. C-10}\cite{xia2020part}. The setup here is same as the CIFAR-100 dataset with 30\% noise.
\textbf{Stanford Cars} \cite{krause20133d}. This dataset contains around 16k images from 102 categories of cars based on Model, company, etc. with a roughly equal split between train and test set. We use 1K images (around 10 images per class), from the train set use as our validation set. \\
\textbf{Aircraft} \cite{maji2013fine}. It consists of 10.2k images belonging to different types of aircrafts covering 102 categories. The task involves fine-grained image classification into these 102 classes. The train, validation and test sets are equal splits of data. We combine the train and validation data for applying our scheme and limit the length of our validation set to be the length of original validation set. \\
\textbf{Oxford-IIIT pet dataset} \cite{parkhi2012cats}. It consists of 37 categories representing breeds of dogs and cats, with 200 images per category equally divided into train and test sets. We test our method on the fine-grained image classification task for this dataset. For the validation set, we use 600 examples from the train set. \\
\textbf{ImageNet-A}\cite{hendrycks2021natural}. It comprises of real-world naturally existing (unmodified) examples which have been mostly misclassified by ResNet models. It has been proposed as a test set, comprising 7500 images belonging to 200 classes of the ImageNet-1K dataset.\\
\textbf{ImageNet-R}\cite{hendrycks2021many}. 
It comprises Images styled to various artistic renditions like paintings, drawings, etc. belonging to subset of classes of the ImageNet-1K dataset. It was basically designed to test generalization onto such renditions as the ImageNet dataset is restricted to photos. It consists of 30k images belonging to 200 classes.\\
\textbf{Camelyon Dataset} \cite{bandi2018detection}. It consists from training data from various sources treated as different domains and a test dataset from completely different sources or domains. It involves the task of classification of breast cancer patients into various stages. \\
\textbf{iWildCam Dataset} \cite{beery2020iwildcam}. It is again a classification task with the train dataset consisting of 441 different locations and a total of 217k images and the test dataset consist of a disjoint set of 111 different locations and a total of 63k images, spread throughout the globe. The aim is to identify the animal species from the given image. \\
\textbf{Diabetic Retinopathy} \cite{kaggle}. It involves the classification of a given retina scans into various levels of diabetic retinopathy (total 5 levels). For this work, we have binarized the task to 2 categories : (0,1) and(2,3,4). We train the model on the Kaggle dataset extracted from hospitals in the US. For OOD testing, we use the APTOS dataset \cite{aptos}, extracted from  a different country's hospital, resulting in a significant domain shift presumably due to equipment and protocol differences. The test set consists of around 3k images and val set around 3k images.\\

\begin{table*}[!htb]
\vskip 0.15in
\begin{center}
\begin{small}
\begin{sc}
\begin{tabular}{lcccccccccccr}
\toprule
Data set & Easy & Hard & Random & Opt &ERM & MBW & RHO-Loss & FSR & MWN & MAPLE & BiLAW\\
\midrule
CIFAR-100  &79.13 &81.17 &80.02 & 81.42&79.45 &79.22 &79.63 &80.10 &  79.13
&80.12
&79.87
\\

  Aircraft &  80.87
&81.28
&80.69 &
81.78&80.22
&80.12
&80.34
&80.55 &  80.11
& 80.58
& 80.37

\\
  Stanford Cars &  80.22
&82.23
&81.35 &
82.47&80.32
&80.67
&80.48
&80.55 &  80.11
&81.40
&81.07

\\ 
  ImageNet-100 &  86.82
&87.95
&87.62 &
87.67&86.91
&86.34
&86.96
&87.18 &  86.78
& 87.67
& 87.25

\\
  ImageNet-1K &  74.17
&76.26
&74.88 &
76.61&75.65
&75.23
&75.13
&75.76 & 75.12
&75.35
&75.60

\\
  Oxford-IIIT  &91.15
&92.72
&92.38 &
93.09&92.33
&92.18
&92.35
&92.51 &  92.04
&92.17
&92.38

 \\

DR (In-Dist) & 89.86
&91.78
&91.00 & 91.89
&90.65
&90.80
&90.74
&90.91 &  90.72 & 91.06
& 91.12
\\

\bottomrule
\end{tabular}
\end{sc}
\end{small}
\end{center}
\caption{Comparison of accuracies of LRW-Hard/Easy/Random and the existing baseline-reweighting/data selecting methods along with the standard ERM classifier, on various datasets discussed in the paper.}
\label{accuracy-table}
\end{table*}

\begin{table*}[!htb]
\vskip 0.15in
\begin{center}
\begin{small}
\begin{sc}
\begin{tabular}{lcccccccccccr}
\toprule
Data set & Hard & Easy & Random & Opt &MAPLE & StableNet & RHO-Loss & FSR & MWN & MBR & ERM\\
\midrule
Camelyon &
71.06
&70.13
&70.22
&71.43
&70.35
&70.31
&71.12
&70.34
&70.06
&70.46
&70.22

\\

  iWildCam 
  &72.56
&72.12
&71.46
  &72.68
&71.59
&71.52
&71.13
&71.45
&71.02
&71.40
&71.32
\\
  ImageNet-A 
   &5.6
&4.9
 &5.2
  & 5.5
&5.4
&5.4
&5.4
&5.5
&5.2
&5.2
&5.3

\\ 
  DR (OOD) &
  86.9
&85.8
&86.1
  &86.8
&86.2
&86.3
&86.2
&86.2
&85.9
&85.9
&86.1

\\

\bottomrule
\end{tabular}
\end{sc}
\end{small}
\end{center}
\caption{Comparison of accuracies of LRW-Hard/Easy/Random and the existing baseline-reweighting/data selecting methods along with the standard ERM classifier, on various Out-of-Distribution benchmarks discussed in the paper.}
\label{accuracy-table-OOD}
\end{table*}

\begin{figure*}[!htb]
    \centering
    \includegraphics[width=0.32\textwidth]{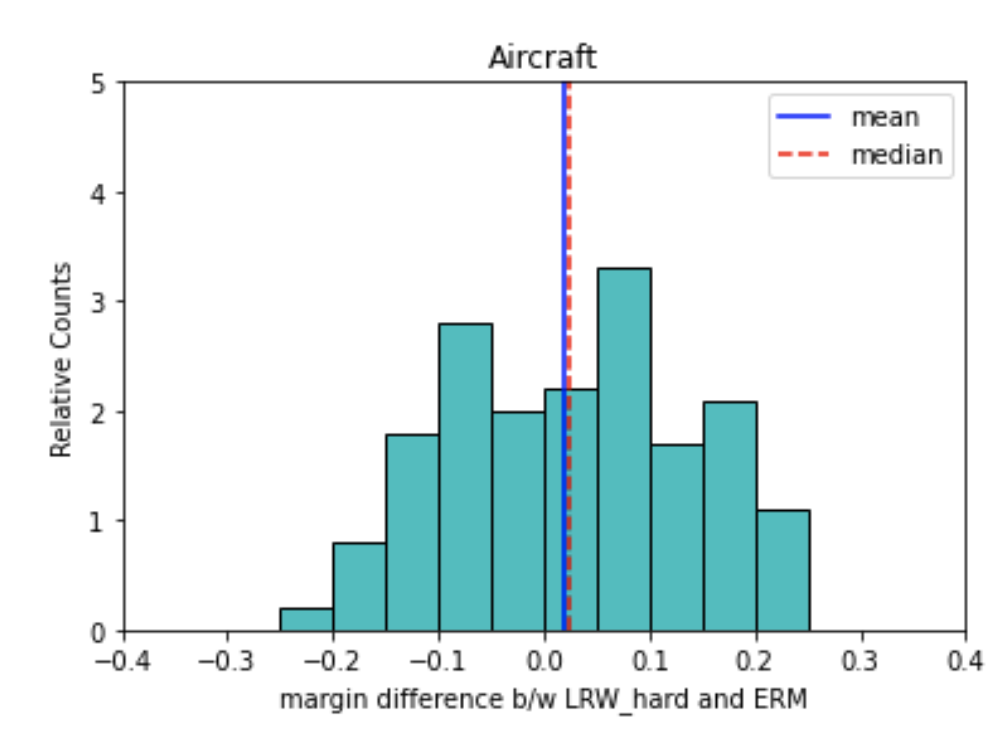}
    \includegraphics[width=0.32\textwidth]{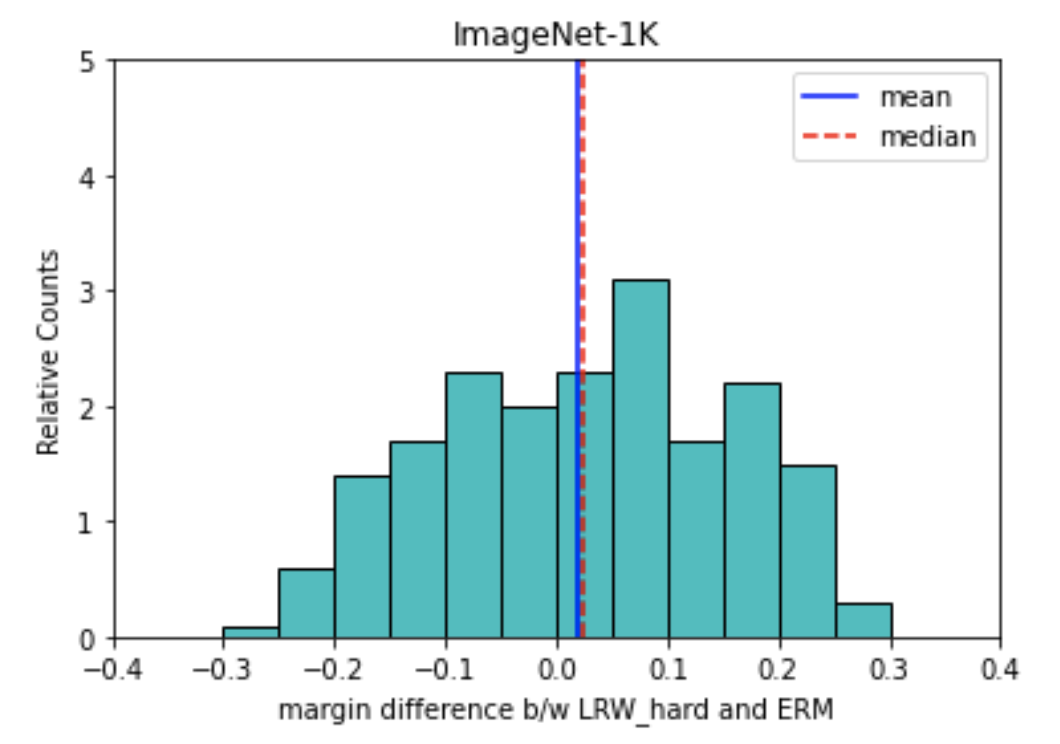}
    \includegraphics[width=0.32\textwidth]{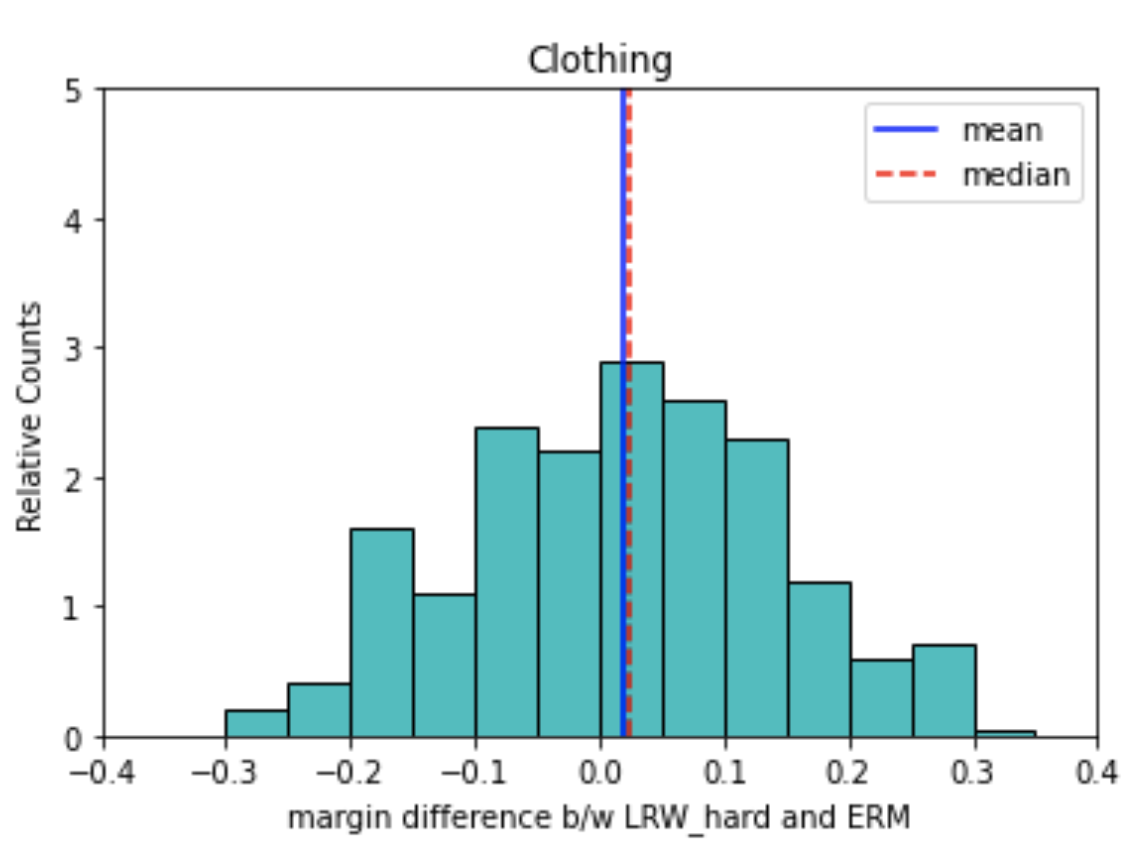}
    \includegraphics[width=0.32\textwidth]{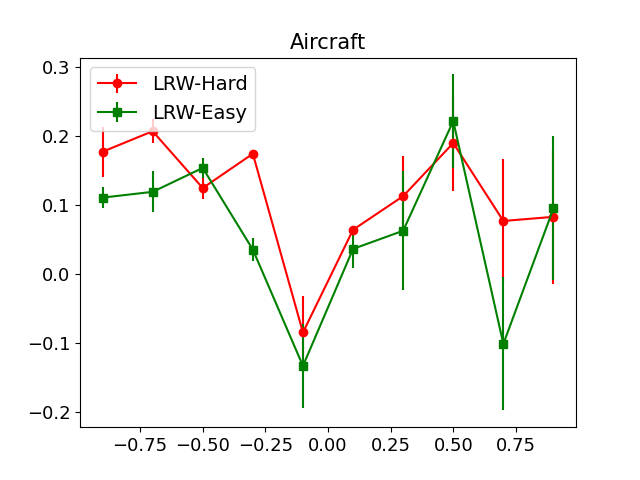}
    \includegraphics[width=0.32\textwidth]{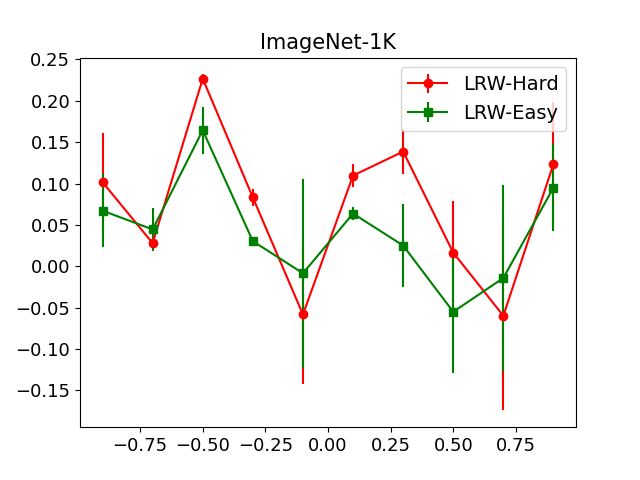}
    \includegraphics[width=0.32\textwidth]{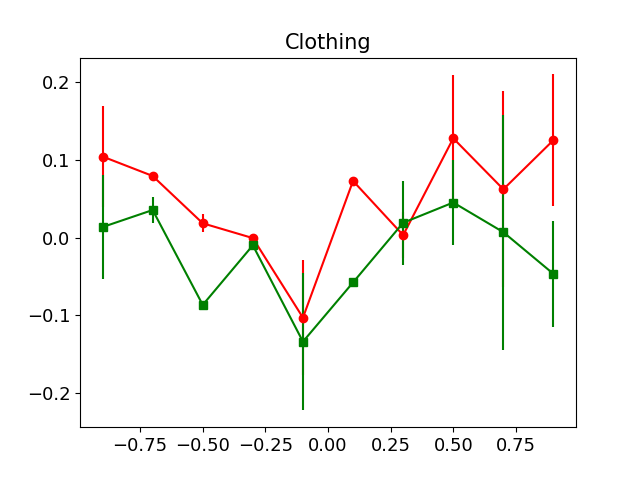}
    \caption{\textbf{Top.} Histograms of difference in margin of the LRW-Hard trained classifier and ERM classifier. \textbf{Bottom.} Mean and Standard deviation of margin deltas between the LRW Hard/Easy methods and the ERM classifier on the test examples, binned by ERM classifier margins with a bin width of 0.2 units. On the x-axis is the starting of the margin interval bin of the ERM classifier. We see that LRW-Hard classifiers have a tendency to \textit{increase} margin over and above the ERM margin, whereas LRW-easy classifiers appear to reduce margin.}
    \label{fig:mean_std_rem}
\end{figure*}

\subsection{Baselines}

\textbf{ERM.} This is the standard empirical risk minimization classifier, obtained by training a classifier with cross-entropy loss in a conventional batch-learning manner, and validation data used for hyperparameter tuning. Its margins are then used for selecting the validation set for our method.\\
\textbf{Margin-based re-weighting.} \citet{liu2021probabilistic} suggest the use of probabilistic margin as an ad-hoc scaling term for adversarial samples, in an adversarial training regime. We extended this method to simply reweight all data points according to ERM margin, and training a second classifier from scratch. This is a control baseline to contrast the contributions of the LRW framework against the notion of using margin directly in the training loss. \\
\textbf{Meta-Weight-Net.} \citet{shu2019meta} proposed a learned re-weighting scheme based on bi-level optimization introduced in \cite{ren2018learning} but using a 2-layer meta network to predict instance level weights using taking loss as input.  \\
\textbf{Fast-Sample Re-weighting.} \citet{zhang2021learning} proposed a meta-learning scheme which generates a pseudo validation set and then proposed an efficient and faster sample re-weighting based meta-learning technique to re-weight train examples.\\
\textbf{RHO-Loss}. \citet{mindermann2022prioritized} proposes selecting examples based on a Reducible Holdout Set Loss to maximize generalization--this results in an implicit reweighting of training instances. We use our validation set as its holdout set for all the datasets. \\
\textbf{MAPLE}. Proposed in \cite{zhou2022model} Similar to \cite{ren2018learning, zhang2021learning} based on free-parameters based reweighting setup but advocates for certain "OOD risk" objectives for improving generalization/robustness of the classifier.\\
\textbf{BiLAW}. Another bi-level optimization based learned re-weighting similar to Meta-weight-Net, proposed by \citet{holtz2022learning}, which feeds multi-class margin to the meta-network, resulting in a more robust (both in-dist and adversarially) classifier. \\
\textbf{StableNet}. Proposed by \citet{zhang2021deep} and optimizes sample weights such that the dependence among various features is decreased. \\
\textbf{GDW}. A recently proposed \cite{chen2021generalized} re-weighting method designed for handling skewed or noisy label scenarios.

\section{Time Complexity, Compute, Tuning:}
\textit{Training time:} Averaging over datasets in Figure \ref{fig:acc_gain}, runtime as a function of ERM cost is (LRWOpt, LRW-hard, MWN, L2R) := (1.6x, 2.4x, 1.4x, 1.4x). LRWOpt is marginally more expensive than MWN for noticeably higher accuracy, and substantially lower than the train-twice heuristic (LRW-hard) while meeting or exceeding its accuracy.\\
\textit{FLOPS:} (LRWOpt, LRW-hard) are $\sim$ (1.7x, 2.3x) ERM. \\
\textit{Hyperparams:} Compared to LRW, we have one additional tunable hyperparameter for splitter's learning rate. LRW itself requires a meta-network learning rate and $Q$. Sensitivity analysis suggests any moderate value of $Q$ is sufficient; we set $Q=5$ across datasets. The parameter $\delta$ is fixed at 0.1; we believe this is a reasonable general tradeoff between training \& validation sizes. We also found that ERM hyperparameters are sufficient for LRW; the meta-network and splitter learning rates can be tied to classifier learning rates without much degradation.

\section{Comparison with \textit{only} hard examples in validation set} 
We further analyze three new variants which involve using only hard examples in the validation set: First we incorporate the loss highlighted in \cite{li2020difficulty} for the validation set along with our LRW-Hard method. Second, in our LRWOpt method we decrease the threshold $\Theta$ for train set to 0.2, such that only the hardest of the examples are there in the validation set and third where in the LRW-hard, we limit the validation set to only negative margin (i.e., incorrectly classified) examples from ERM. Table \ref{tab:gain_new} shows accuracy \% gains of LRWOpt over these variants on 4 randomly picked datasets including 1 OOD challenge.

\begin{table}[h]
    \centering
    \footnotesize
    \begin{tabular}{c|c|c|c|c}
    \toprule
        & IN-100 & DR & CIFAR-100 & iWildCam \\
        \midrule
       Variant 1  &  0.8 & 0.6 & 1.1 & 0.8\\
       Variant 2  & 0.7 & 0.9 & 1.3 & 0.6\\
       Variant 3  & 0.7 & 0.7 & 1.5 & 0.9\\
       \bottomrule
    \end{tabular}
    \caption{Accuracy gain \% of LRWOpt over variants.}
    \label{tab:gain_new}
\end{table}
\vspace{-0.1in}
\noindent

\section{Accuracy Comparison}
We present the raw accuracy values of all methods reported in the paper. Table \ref{accuracy-table} shows the results for this comparison corresponding to Figure 1 and Table \ref{accuracy-table-OOD} corresponding to Figure 2 in the main paper.  This underscores shows the effectiveness of our method as it shows gains even at high accuracy values like for the Oxford-IIIT pets dataset. Furthermore, it is also effective on relatively difficult datasets where model suffer in performance like Imagenet-1K dataset.

\section{Analysis of Predicted Margins}
We show the histograms of difference in margins predicted by LRW-Hard and ERM classifier on the remaining datasets including Aircraft, Stanford Cars and ImageNet-100. Figure \ref{fig:mean_std_rem} shows the results. Here also, a pattern similar to the main draft is observed, \textit{i.e.}, more examples are on the positive side, thus showing that LRW-Hard is able to optimize margins. This is supported by both mean and median being on the positive side. \\
We also show experiment involving grouping the points based on ERM margin values and reporting the mean and std of the margin difference between LRW Hard/Easy and ERM classifier, for the remaining datasets including Aircraft, ImageNet-100 and ImageNet-1K, in Figure \ref{fig:mean_std_rem}. The results are similar as in the main paper for the CIFAR-100 and Clothing datasets. Here also, for LRW-Hard, the mean is positive and significant, especially for positive margin examples showing the effectiveness of LRW Hard compared with ERM. Furthermore, the difference between LRW Hard and LRW Easy plots is significant, further backing the claim regarding importance of validation set and its effectiveness in margin maximization.


\end{document}